\documentclass[letterpaper, 10pt, reqno]{amsart}

\usepackage[usenames,dvipsnames,svgnames,table]{xcolor}
\usepackage{amsmath, amsthm, amssymb, bbm, bm, enumerate}
\usepackage{graphicx}
\usepackage{float}      
\usepackage{wrapfig, subcaption}
\usepackage[margin=0cm]{caption}
\usepackage[titletoc, toc]{appendix}
\usepackage{microtype}
\usepackage{blindtext}
\usepackage{tabularx}

\usepackage[boxruled, vlined, linesnumbered]{algorithm2e}
\SetAlFnt{\small}
\SetAlCapFnt{\small}
\SetAlCapNameFnt{\small}
\usepackage{algorithmic}
\algsetup{linenosize=\tiny}

\usepackage[bookmarks=false]{hyperref}
\hypersetup{
colorlinks=true, linkcolor=Sepia, citecolor=Sepia, filecolor=magenta, urlcolor=NavyBlue,
}
\usepackage{url, bookmark}

\newcommand{\reals}{\mathbb{R}}

\newcommand{\abs}[1]{\ensuremath \left| #1 \right|}
\newcommand{\norm}[1]{\ensuremath \lVert#1\rVert}

\newcommand{\ag}[1]{\ensuremath \left\langle#1\right\rangle}
\providecommand{\OO}{\mathcal{O}}

\newcommand{\trace}{\trm{tr}}

\newcommand{\aeq}[1]{\begin{align} #1 \end{align}}
\newcommand{\aeqs}[1]{\begin{align*} #1 \end{align*}}
\newcommand{\beq}[1]{\begin{equation}#1\end{equation}}

\newcommand{\trm}[1]{\mathrm{#1}}

\newcommand{\la}{\leftarrow}

\providecommand\f[2]{\ensuremath \frac{#1}{#2}}
\providecommand\rbrac[1]{\ensuremath \left(#1\right)}
\providecommand\sqbrac[1]{\ensuremath \left[#1\right]}
\providecommand\cbrac[1]{\ensuremath \left\{#1\right\}}

\newtheorem{theorem}{Theorem}

\newtheorem{lemma}[theorem]{Lemma}

\theoremstyle{definition}

\newtheorem{remark}[theorem]{Remark}

\renewcommand{\P}{\trm{P}}
\newcommand{\E}{\mathbb{E}}
\providecommand{\ones}{\mathbbm{1}}
\providecommand{\ind}{{\bf 1}}

\renewcommand{\implies}{\Rightarrow}

\newcommand{\s}{\sigma}

\renewcommand{\r}{\rho}
\renewcommand{\t}{\tau}
\renewcommand{\th}{\theta}
\renewcommand{\a}{\alpha}

\newcommand{\e}{\epsilon}

\newcommand{\g}{\gamma}
\renewcommand{\d}{\delta}
\newcommand{\D}{\Delta}

\renewcommand{\l}{\lambda}

\newcommand{\Th}{\Theta}

\def \OO {\mathcal{O}}

\newcommand{\var}{\trm{var}}

\newcommand{\Berdist}{\trm{Ber}}

\newcommand{\zoo}{\{0, 1\}}

\newcommand{\W}[1]{W^{#1}}
\newcommand{\h}[1]{h^{#1}}

\newcommand{\crt}{\trm{crt}}
\newcommand{\tH}{\widetilde{H}}

\newcommand{\ip}{{i_1, i_2, \ldots, i_p}}
\newcommand{\Jip}{J_{i_1, \ldots, i_p}}
\newcommand{\Jlp}{J_{l_1, \ldots, l_p}}
\newcommand{\Kip}{K_{i_1, \ldots, i_p}}
\newcommand{\Wip}{W_{i_1, \ldots, i_p}}
\newcommand{\gip}{{\g_{i_1, \ldots i_p}}}

\newcommand{\sip}{{\s_{i_1} \ldots \s_{i_p}}}

\newcommand{\Hnp}{H_{n,p}}
\newcommand{\hY}{\widehat{Y}}

\newcommand{\data}{\trm{input}}
\newcommand{\convolution}{\trm{conv}}

\newcommand{\meanpool}{\trm{meanpool}}

\newcommand{\dropout}{\trm{dropout}}
\newcommand{\linear}{\trm{linear}}
\newcommand{\softmax}{\trm{softmax}}

\newcommand{\block}{\trm{block}}

\newcommand{\mnistfc}{\trm{mnistfc}}

\newcommand{\tf}{\widetilde{f}}
\newcommand{\lf}{\overline{f}}
\usepackage{mathptmx}
\usepackage{natbib}

\graphicspath{{fig/}}
\usepackage[margin=1.5in]{geometry}

\newcommand{\ignore}[1]{}

\title[On the Energy Landscape of Deep Networks]{\centering On the Energy Landscape of Deep Networks}
\author{Pratik Chaudhari}
\thanks{Computer Science Department, University of California, Los Angeles}

\author{Stefano Soatto}
\thanks{Email: \href{mailto:pratikac@ucla.edu}{pratikac@ucla.edu}, \href{mailto:soatto@ucla.edu}{soatto@ucla.edu}}

\begin{document}

\begin{abstract}
We introduce ``AnnealSGD'', a regularized stochastic gradient descent algorithm motivated by an analysis of the energy landscape of a particular class of deep networks with sparse random weights. The loss function of such networks can be approximated by the Hamiltonian of a spherical spin glass with Gaussian coupling.
While different from currently-popular architectures such as convolutional ones, spin glasses are amenable to analysis, which provides insights on the topology of the loss function and motivates algorithms to minimize it. Specifically, we show that a regularization term akin to a magnetic field can be modulated with a single scalar parameter to transition the loss function from a complex, non-convex landscape with exponentially many local minima, to a phase with a polynomial number of minima, all the way down to a trivial landscape with a unique minimum. AnnealSGD starts training in the relaxed polynomial regime and gradually tightens the regularization parameter to steer the energy towards the original exponential regime.
Even for convolutional neural networks, which are quite unlike sparse random networks, we empirically show that AnnealSGD improves the generalization error using competitive baselines on MNIST and CIFAR-10.
\end{abstract}

\maketitle

\section{Introduction}
\label{s:intro}

Complex deep network architectures --- Inception~\citep{szegedy2014going}, ResNets~\citep{he2016identity}, generative adversarial nets~\citep{goodfellow2014generative}, value-policy networks~\citep{silver2016mastering}, to name a few --- are hard to train~\citep{sutskever2013importance}. The work-horse of deep learning, stochastic gradient descent (SGD)~\citep{bottou1998online}, therefore employs a number of techniques such as momentum~\citep{graves2013generating}, weight decay, adaptive step-sizes~\citep{duchi2011adaptive,kingma2014adam} typically in conjunction with a hyper-parameter search~\citep{bergstra2012random}. These techniques are applicable to general non-convex problems, and we would like to improve their performance when used to train deep networks. This has motivated attempts to understand the error landscape of various models of neural networks, e.g., deep linear networks~\citep{saxe2013exact}, deep Gaussian processes~\citep{duvenaud2014avoiding}, spin glasses~\citep{choromanska2014loss}, tensor factorization~\citep{haeffele2015global,janzamin2015beating}, robust ensembles~\citep{baldassi2016unreasonable} etc. However, constructing efficient versions of SGD that exploit these insights still remains a challenge and this is our primary motivation in this paper.


We consider a model for deep networks with random, sparse weights; we show that the loss function of such a network resembles the Hamiltonian of a spin glass which is a well-studied model in statistical physics~\citep{talagrand2003spin}. Existing results for this model point to a complex energy landscape with exponentially many local minima and saddle points~\citep{auffinger2013random}. Empirically, this manifests in deep networks as sensitivity of the training procedure to initialization and learning rates~\citep{sutskever2013importance,singh2015layer} or even divergence for complex architectures~\citep{cai2016unified}. The energy landscape also has a layered structure, i.e., high-order saddle points with many negative eigenvalues of the Hessian lie at high energies while saddle points with few descending directions are located at low energies near the ground state. Local minima proliferate the energy landscape: there are exponentially many of them, at all energy levels.

An effective way of modifying this landscape is to perturb the Hamiltonian by a random magnetic field~\citep{fyodorov2013high} or another uncorrelated Hamiltonian~\citep{talagrand2010mean}. In this paper, we employ the former; this is an additive term in the gradient of the loss function. In fact, there exists a critical threshold of the perturbation below which the exponential regime persists and above which the landscape trivializes to leave only one local minimum. This phase transition is not sharp, i.e., that there exists a small band around the perturbation threshold that results in polynomially many local minima. Using a control parameter, we can thus smoothly transition between the original complex landscape and a completely degenerate loss function. Moreover, we prove that such an annealing does not change the locations of the local minima of the original problem.

SGD can readily benefit from this phenomenon, indeed we can compute the exact magnitude of the perturbation to start training in the relaxed polynomial phase and sequentially reduce the perturbation to tighten towards original landscape. This leads SGD to train faster in the beginning thanks to the more benign error landscape. This is prominently seen in our experiments on fully-connected and convolutional deep networks; AnnealSGD alleviates the vanishing gradient problem for the former and results in a better generalization error on the latter.

\section{Related work}
\label{s:related_work}

Gradient noise is a very effective technique in non-convex optimization e.g., SGD can escape from strict saddle points~\citep{ge2015escaping}, and indeed, stochasticity of gradients is crucial to training complex neural networks. However, the mini-batch gradient often does not have high variance in all directions which necessitates additional gradient noise~\citep{jim1996analysis,nalisnick2015scale} and makes SGD highly dependent on the dimensionality~\citep{lee2016gradient}. Our proposed algorithm, AnnealSGD, instead fixes a perturbation in the beginning and reduces its strength as training progresses. We also analyze the effect of such a perturbation by interpreting it as an external magnetic field in a spin glass model. In fact, our connection to spin glasses shows that gradient noise is detrimental because it forces the weights to align in different uncorrelated directions at successive iterations (cf. Figs.~\ref{fig:cifarconv_align}) leading to higher error rates (cf. Table~\ref{tab:cifar10}). In our experience, gradient noise also tends to be more difficult to tune.

Our work is closely related to~\citet{choromanska2014loss} which discusses the exponential regime for a spin glass model of dense, deep networks under assumptions where paths in the network connected to different input neurons are independent. We instead employ sparsity --- indeed, almost 95\% weights in a typical deep network are near-zero~\citep{denil2013predicting, chen2015compressing} --- and show that dependent paths do exist but their contribution to the loss function is small. Our random, sparse model for deep networks in Sec.~\ref{ss:model_deep_networks} is motivated from~\citet{arora2013provable} and leverages upon its results. These assumptions enable a systematic inquiry into the energy landscape of deep networks are not any more restrictive than other approaches in literature.

A follow-up work to~\citet{choromanska2014loss} by~\citet{sagun2014explorations} shows that SGD fails to progress beyond a very specific energy barrier, one that can be seen in spin glass computations as the onset of low-order saddle points and local minima. Topology trivialization introduced in this paper demonstrates a way to make such an energy landscape more amenable to first order algorithms.

Conceptually, our approach is similar to homotopy continuation~\citep{allgower2012numerical,mobahi2016training} and sequentially tightening convex relaxations~\citep{jaakkola2010learning,wang2014tighten}. Our theoretical development in Sec.~\ref{s:perturbations_of_hamiltonian} is however specialized to spin glass Hamiltonians, thereby enabling an analysis and a direct connection to deep networks.

\section{Preliminaries}
\label{s:preliminaries}

\subsection{A model for deep networks}
\label{ss:model_deep_networks}

Let us consider a deep neural network with $p$ hidden layers and $n$ neurons on each layer. We denote the observed layer at the bottom by $X \in \zoo^n$ which is generated by a binary feature vector $\xi \in \zoo^n$. More precisely,
\beq{
    \label{eq:h_to_X}
    X = g \rbrac{{\W{1}}^\top\ g \rbrac{ {\W{2}}^\top \ldots\ g \rbrac{ {\W{p}}^\top\ \xi} \ldots }};
}
where $\W{k} \in [-N,\ N]^{n \times n}$ for $k \leq p$ and some constant $N$ are the weight matrices and $g(x) = \ind_{\cbrac{x \geq 0}}$ are threshold nonlinearities. $\ind_A$ denotes the indicator function of the set $A$.
To make this model analytically tractable, we make the following assumptions:
\begin{enumerate}[(i)]
    \item $\xi$ has $\rho$ non-zero entries and $\r \rbrac{\f{d}{2}}^p$ is a constant;
    \item for $d = n^{1/p}$, every entry $\W{k}_{ij}$ is an iid zero-mean random variable with $\P (\W{k}_{ij} > 0) = \P (\W{k}_{ij} < 0) = d/(2n)$ and zero otherwise. This results in an average $d$ non-zero weights for every neuron;
    \item distribution of $\W{k}_{ij}$ is supported on a finite set of size $\Th(n)$ is not all concentrated near zero.
\end{enumerate}
If every entry of $\xi_i$ is an indicator of some class, the first assumption implies that at most $\r$ classes are present in every data sample $X_i$. The Chernoff bound then implies that with high probability, each $X$ has $\r \rbrac{\f{d}{2}}^p$ fraction non-zero entries, which we assume to be constant. The third assumption on the support of the weight distribution ensures that the weights of the deep network are well-conditioned.

The above model enjoys a number of useful properties by virtue of it being random and sparsely connected. For high sparsity ($d < n^{1/5}$ or equivalently, $p > 5$), the network exhibits ``weight-tying'', a technique popular in training denoising auto-encoders, whereby the hidden layer can be computed (w.h.p) by running the network in reverse:
\beq{
    \xi = g \rbrac{\W{p} g \rbrac{ \W{p-1} \ldots\ g \rbrac{\W{1} X - \f{d}{3} \ones_n} - \f{d}{3} \ones_n} \ldots } \notag;
    \label{eq:denoising_autoencoder}
}
where $\ones_n$ is a vector of ones~\citep{arora2013provable}. We exploit this property in Thm.~\ref{thm:deep_net_H} to connect the loss of such a network to the spin glass Hamiltonian.

\textbf{Classification model:} In a typical binary classification setting, one uses a support vector machine (SVM) or a soft-max layer on the feature vector $\xi$ to obtain an output $Y \in \zoo$. We model this as
\beq{
    Y = g \rbrac{\W{p+1} \xi - \f{d}{3} \ones_n}.
    \label{eq:Y}
}

\subsection{Deep networks as spin glasses}
\label{ss:deep_nets_spin_glasses}

A $p$-spin glass for an integer $p \geq 1$ with $n$ spins is given by an energy function, also known as the ``Hamiltonian'',
\beq{
    -\Hnp(\s) = \f{1}{n^{(p-1)/2}} \sum_{\ip =1}^n \Jip \sip;
    \label{eq:p_spin_glass}
}
where $\s = (\s_1, \ldots, \s_n) \in \reals^n$ is a configuration of the ``spins'' and the ``disorder'' $\Jip$ are iid standard Gaussian random variables. Note that $\Hnp(\s)$ is zero-mean and Gaussian for every $\s$. The term $n^{(p-1)/2}$ is a normalizing constant that ensures that the Hamiltonian scales as $\Th(n)$. In spin glass computations, for ease of analysis one typically assumes a spherical constraint on the spins, viz., $\sum_i \s_i^2 = n$ or in other words, $\s \in S^{n-1}(\sqrt{n})$ which is the $(n-1)$-dimensional hypersphere of radius $\sqrt{n}$.

The following theorem connects the model for deep networks in Sec.~\ref{ss:model_deep_networks} to the spin glass Hamiltonian.
\begin{theorem}
\label{thm:deep_net_H}
If the true label $Y^t \sim \mathrm{Ber}(q)$ for some $q < 1$, the zero-one loss $\E_X |\hY - Y^t|$ has the same distribution as
\beq{
    \label{eq:lem_deep_net_H}
    -H_{n,p}(\s) = \f{J}{n^{(p-1)/2}}\ \sum_{\ip = 1}^n\ \Jip \ \sip
}
up to an additive constant. $\Jip$ is a zero-mean, standard Gaussian random variable, $J \in \reals$ is a constant and $\s \in S^{n-1}(\sqrt{n})$.
\end{theorem}
\begin{proof}(sketch, cf. Appendix~\ref{s:proofs})
For each input neuron $X_i$, if $\Gamma_i$ is the set of ``active paths'', i.e., the ones with non-zero activations at each layer, we can write the output $Y$ as
$$
    Y = \sum_{i=1}^n\ \sum_{\g \in \Gamma_i}\ X_i\ W_\g;
$$
where $W_\g$ is the product of the weights along this path. The activation probability of a path $\g$ (consisting of the $i^{\trm{th}}$ input neuron, $i_1^{\trm{th}}$ neuron on the first hidden layer and so on) concentrates for a sparse network. Note that activations are not independent due to layers above. We can use this along with an unwrapping of the first layer to bound the above summation to get the Gaussian coupling term $\Jip$ in~\eqref{eq:lem_deep_net_H}.

We next use the assumption that weights $W^k_{ij}$ are supported on a finite set: the term $W_\g$ can take at most $\OO(n^p)$ distinct values. Denoting a ``spin'' as $\s \in \trm{supp}(W^k_{ij})$ (which is assumed of size $\Th(n)$), we can write $W_\g$ as a product of $p$ terms $\s_{i_1}, \ldots, \s_{i_p}$ and make an approximation that $\s$ lies on the sphere $S^{n-1}(\sqrt{n})$.
\end{proof}

Let us emphasize that the ``spins'' resulting from Thm.~\ref{thm:deep_net_H} are the quantization of weights of a deep network on a set $[-N, N]$ (we can be safety set $N =\sqrt{n}$). This quantization is currently necessary to map the $\OO(n^p)$ distinct paths $W_\g$ into the set of $p$-products of $n$ spins $(\s_1, \ldots, \s_n)$. We also make the standard spin glass assumption that $\s \in S^{n-1}(\sqrt{n})$.

For the sake of building intuition in the sequel, one can simply think of $\s$ as the real-valued weights themselves.

\section{Perturbations of the Hamiltonian}
\label{s:perturbations_of_hamiltonian}

In Sec.~\ref{ss:scaling_critical_points}, we discuss the asymptotic scaling of critical points, i.e., all local minima and saddle points. We then show how the energy landscape changes under an external magnetic field in Sec.~\ref{ss:modifying_landscape} and Sec.~\ref{ss:annealing}. The next section Sec.~\ref{ss:quality_local_minima} describes how this perturbation affects the locations of the local minima.

\subsection{Asymptotic scaling of critical points}
\label{ss:scaling_critical_points}

For any $u \in \reals$, let the total number of critical points in the set $(-\infty, u]$ be
$$
\crt(u) = \abs{\cbrac{\s:\ \nabla H(\s) = 0,\ H(\s) \leq n\ u}}.
$$
We will denote $\crt(\infty)$ by $\crt(H)$. The index of a critical point $\s \in \crt(H)$ is the number of negative eigenvalues of the Hessian of $H_{n,p}(\s)$ and we denote critical points of index $k$ by $\crt_k(u)$ (e.g., local minima are simply $\crt_0(u)$). Using statistics of the eigen-spectrum of the Gaussian Orthogonal Ensemble (GOE), one can show:
\aeq{
    \label{eq:thetak}
    \lim_{n \to \infty}\ \f{1}{n}\ \log \E\ \crt_k(u) &= \Th_k(u),\\
    \label{eq:theta_all}
    \lim_{n \to \infty}\ \f{1}{n} \log \E\ \crt(u) &= \Th(u);
}
for any $k \geq 0$ and $p \geq 2$~\citep{auffinger2013random}. The functions $\Th_k(u)$ and $\Th(u)$ are known as the ``complexity'' of a spin glass. Fig.~\ref{fig:energy_barriers} elaborates upon these quantities.
\begin{figure}[!tbh]
\centering
\includegraphics[width=0.6 \columnwidth]{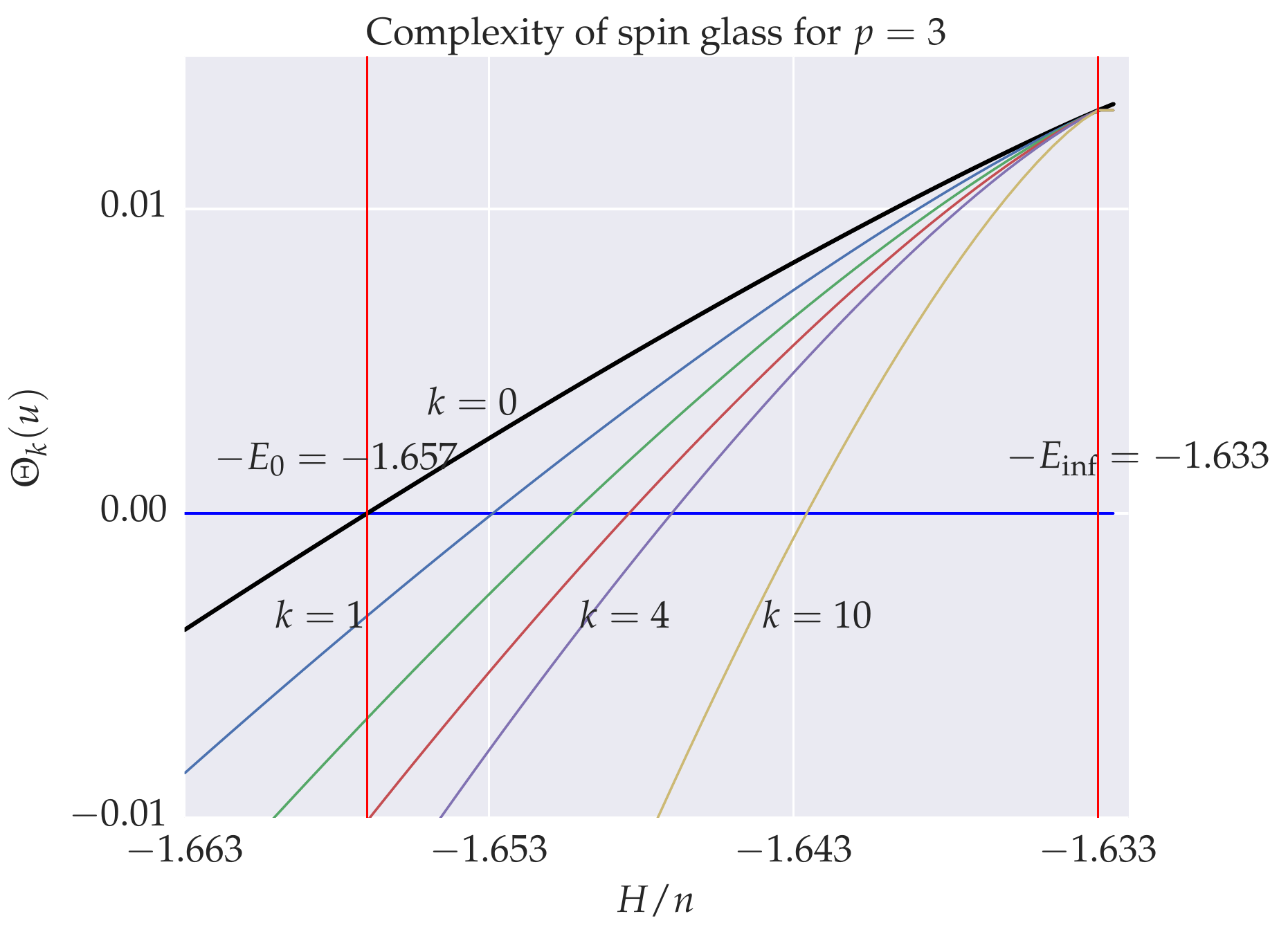}
\caption{The number of critical points in $(-\infty, u]$ scales as $e^{n \Th_k(u)}$. The black curve denotes $\Th_0(u)$, the complexity of local minima. Below $\lim_{n\to \infty}\ \inf H(\s)/n = -E_0$ which is where $\Th_0(u)$ intersects the $x$-axis, there are no local minima in the logarithmic scaling~\citep{auffinger2013random}. Similarly, below the point where $\Th_k(u)$ intersects the $x$-axis, there are no critical points of index higher than $k$. The Hamiltonian thus shows a layered landscape, higher-order saddle points to the right and in fact, local minima concentrated near $-E_0$ on the far left~\citep{subag2015extremal}.}
\label{fig:energy_barriers}
\end{figure}

\subsection{Modifying the energy landscape}
\label{ss:modifying_landscape}

Let now consider a perturbation to the Hamiltonian in~\eqref{eq:lem_deep_net_H}
\beq{
    -\tH(\s) = -H(\s) + h^\top \s;
    \label{eq:tH_general}
}
where $h = (h_1, \ldots, h_n) \in \reals^n$ is the external perturbation. We let $h_i$ be iid zero-mean Gaussian random variables with
$$
    h_i \sim\ \trm{N}(0,\ \nu^2)\quad \trm{for\ all}\ i \leq n.
$$
We now define a parameter
\beq{
    B = \f{J^2 p(p-2) - \nu^2}{J^2 p^2 + \nu^2} \in (-1,\ 1 - 2/p].
    \label{eq:B_tH}
}
that governs the transition between the exponential, polynomial and trivialized regimes. The value $B=0$, i.e., $\nu_c := J \sqrt{p(p-2)}$ demarcates the exponential regime from total trivialization.
\begin{theorem}[\citet{fyodorov2013high}]
\label{thm:transition_critical_dot}
For large $n$, the expected number of critical points of the perturbed Hamiltonian is given by
\beq{
    \E\ \mathrm{crt}(\tH) = \begin{cases}
    2 & \trm{if}\ B = -\Omega(n^{-1}),\\
    \f{2n}{\sqrt{\pi}} \t^{-3/2} & \trm{if}\ B = -\f{\t}{n},\\
    4 n^{1/2}\ \sqrt{\f{1+B}{\pi B} } \exp \rbrac{ \f{n}{2} \log \f{1+B}{1-B} } & \trm{if}\ B > 0.
    \end{cases}
    \label{eq:transition_critical_dot}
}
where $\tau \ll n$ is a constant.
\end{theorem}
For $\nu < \nu_c$, i.e., $B > 0$ in the third case, we have the (non-asymptotic version of) exponential regime in~\eqref{eq:thetak} and~\eqref{eq:theta_all}. The first case shows total trivialization because any smooth function on the sphere $S^{n-1}(\sqrt{n})$ has at least one minimum and at least one maximum, thus there is exactly one minimum if the perturbations are larger than $\nu_c = J \sqrt{p(p-2)}$. Indeed, such a trivialization occurs for all regularized loss functions when the regularization term overwhelms the data term. However, for most loss functions unlike spin glass Hamiltonians, we cannot compute how the regularization term affects the solution landscape; an exception to this is LASSO where the Least Angle Regression (LARS) algorithm exploits the piecewise-linear solution paths to efficiently compute the best regularization coefficient~\citep{efron2004least}.

More importantly for us, the phase transition between the exponential and trivial regime is not sharp --- otherwise there would be no way of modifying the complexity of the landscape without destroying the loss function. For a band of size $\OO(n^{-1})$ below zero, as shown in Thm.~\ref{thm:transition_critical_dot}, the spin glass has only $\OO(n)$ critical points. A more elaborate analysis on a finer scale, viz., using an order parameter $B = -\tau/n$ for a constant $\abs{\tau} \ll n$ gives the following theorem.

\begin{theorem}[\citet{fyodorov2013high}]
\label{thm:tau_annealing}
For $B = -\tau/n$, the expected number of critical points is given by
\beq{
    \lim_{n \to \infty}\ n^{-1}\ \E\ \crt(\tH) =
    \begin{cases}
        \f{2}{\sqrt{\pi}}\ \tau^{-3/2} & \trm{for}\ \tau \gg 1,\\
        \f{2}{\sqrt{\pi \abs{\tau}}}\ 2 e^{\abs{\tau}} & \trm{for}\ \tau \ll -1
    \end{cases}
    \label{eq:edge_scaling_tau}
}
\end{theorem}
Note that the $\tau \gg 1$ case in~\eqref{eq:edge_scaling_tau} gives the polynomial regime in~\eqref{eq:transition_critical_dot} with $\OO(n)$ critical points while the second case recovers the $B > 0$ limit of the exponential regime in~\eqref{eq:transition_critical_dot}. Indeed, as Thm.~\ref{thm:tau_annealing} shows, the polynomial band $\abs{B} = \OO(n^{-1})$ extends on either size of $B = 0$, the exponential regime for $B = \Omega(n^{-1})$ lies to the right and the trivialized regime lies to the left of this band with $B = -\Omega(n^{-1})$.

\subsection{Scaling of local minima}
\label{ss:scaling_local_minima}

The expected number of local minima of the Hamiltonian also undergoes a similar phase transition upon adding an external perturbation with the same critical threshold, viz., $\nu_c = J \sqrt{p (p-2)}$. If $\nu < \nu_c$, the energy landscape is relatively unaffected and we have exponentially many local minima as predicted by~\eqref{eq:thetak} and exactly one minimum under total trivialization. Similar to Thm.~\ref{thm:transition_critical_dot}, there is also an ``edge-scaling'' regime for local minima that smoothly interpolates between the two.
\begin{theorem}[\citet{fyodorov2013high}]
\label{thm:kappa_annealing}
For $B = -\f{\kappa}{2} n^{-1/3}$ and some constant $C > 0$, the expected number of local minima is given by
\beq{
    \lim_{n \to \infty}\ \E\ \crt_0(\tH) =
    \begin{cases}
        1 & \trm{for}\ \kappa \gg 1,\\
        C \exp \rbrac{\f{\kappa^2}{24} + \f{4 \sqrt{2} \abs{\kappa}^{3/2}}{3}} & \trm{for}\ \kappa \ll -1.
    \end{cases}
    \label{eq:edge_scaling_kappa}
}
\end{theorem}
First, note that this suggests $\OO \rbrac{\exp (\log n)^2}$ (quasi-polynomial) local minima for $\kappa = \log n$. Also, the band around $B = 0$ for polynomially many local minima is wider; it is $\OO(n^{-1/3})$ as compared to $\OO(n^{-1})$ for critical points. Above the critical perturbation threshold, in particular, between $-n^{-1/3} \lesssim B \lesssim -n^{-1}$ this shows an interesting phenomenon wherein almost all critical points are simply local minima. This has been previously discussed and empirically observed in the deep learning literature~\citep{dauphin2014identifying}. A more precise statement is proved for spin glasses in~\citep{fyodorov2007replica}.

\subsection{Annealing scheme}
\label{ss:annealing}

If we set $B = -\tau/n$ in ~\eqref{eq:B_tH}, for $p \gg 1$ resulting in $\nu_c \approx Jp$, we have
\beq{
    \nu = \nu_c \rbrac{1 + \f{2 \tau}{n} }^{1/2}
    \label{eq:nu_polynomial}
}
for the perturbation to be such that our random, sparse deep network has polynomially many critical points. For $0 < \tau \ll n$, we have $\nu > \nu_c$ but the spin glass is still within the polynomial band whereas large values of $\tau$ result in total trivialization. For large negative values of $\tau$, the external magnetic field is almost zero and we have exponentially many local minima as given by the second case in Thm.~\ref{thm:tau_annealing}.

We can now exploit this phenomenon to construct an annealing scheme for $\tau$. For the $i^{\trm{th}}$ iteration of SGD, we set
\beq{
    \tau(i) = n \rbrac{e^{-i/\tau_0} - \f{1}{2}}.
    \label{eq:tau_exponential_annealing}
}
In practice, we treat $\nu_c$ and $\tau_0$ as tunable hyper-parameters. Considerations of different values of $\tau(0)$, $\lim_{k \to \infty} \tau(k)$ and speed of annealing can result in a variety of annealing schemes; we have used exponential and linear decay to similar experimental results. Algorithm~\ref{alg:annealsgd} provides a sketch of the implementation for a general deep network. Note that Line 10 can be adapted to the underlying implementation of SGD, in our experiments, we use SGD with Nesterov's momentum and ADAM~\citep{kingma2014adam}.

\begin{center}
\begin{minipage}{0.7 \textwidth}
\IncMargin{0.04in}
\begin{algorithm}[H]
    \small
    Input:\hspace{0.35in} $\trm{weights\ of\ a\ deep\ network}\ w$\;
    \vspace{0.01in}
    Estimate:\hspace{0.17in} $\trm{\#layers}\ p$, $\trm{\#neurons}\ n$\;
    \vspace{0.01in}
    Parameters:\hspace{0.05in} $\trm{constant}\ \tau_0$, $\trm{coupling}\ J$, $\trm{learning\ rate}\ \eta$\;
    \vspace{0.02in}
    $h \sim J\sqrt{p(p-2)}\ \trm{N}(0,\ I)$\;
    \vspace{0.01in}
    $i = 0$\;
    \vspace{0.01in}
    \While{$\trm{True}$}
    {
        $X_i \la \trm{sample\ batch}$\;
        $\tau_i \la n \rbrac{e^{-i/\tau_0} - \f{1}{2}}$\;
        \vspace{0.01in}
        $h_{\trm{curr}} \la \rbrac{1 + \f{2 \tau_i}{n}}^{1/2}\ h$\;
        \vspace{0.01in}
        $w\ \la w -\eta \rbrac{\nabla \trm{loss}(w,\ X_i) + h_{\trm{curr}}}$\;
        \vspace{0.01in}
        $i \la i + 1$
    }
    \caption{AnnealSGD}
    \label{alg:annealsgd}
\end{algorithm}
\DecMargin{0.04in}
\end{minipage}
\end{center}

\subsubsection{Annealing is tailored to critical points}
Recent analyses suggest that even non-stochastic gradient descent always converges to local minima instead of saddle points~\citep{lee2016gradient,choromanska2014loss}. It seems contradictory then that we have used the scaling for critical points derived from Thm.~\ref{thm:tau_annealing} instead of the corresponding local minima version, viz., Thm.~\ref{thm:kappa_annealing} which would give
\beq{
    \nu = \nu_c \rbrac{1 + \f{\kappa\ \log n}{n^{1/3}} }^{1/2};
    \label{eq:nu_polynomial_kappa}
}
surely if SGD never converges on saddle points, reducing the complexity of local minima is more fruitful. We are motivated by two considerations to choose the former. Firstly, in our experiments on convolutional neural networks~\eqref{eq:nu_polynomial_kappa} does not perform as well as~\eqref{eq:nu_polynomial}, probably due to the asymptotic nature of our analysis. Secondly, even if SGD never converges to saddle points, training time is marred by the time that it spends in their vicinity. For instance, Kramers law argument for Langevin dynamics predicts that the time spent is inversely proportional to the smallest negative eigenvalue of the Hessian at the saddle point (see~\citet{bovier2006metastability} for a precise version).

\subsection{Quality of perturbed local minima}
\label{ss:quality_local_minima}

Sec.~\ref{ss:modifying_landscape} describes how the perturbation $h$ changes the number of local minima and critical points. It is however important to characterize the quality of the modified minima, indeed it could be that even if there are lesser number of local minima the resulting energy landscape is vastly different. In this section, we show that the answer to this question is negative --- the locations of the local minima in the polynomial regime are only slightly different than the their original locations and this difference vanishes in the limit $n \to \infty$.

We construct a quadratic approximation to the unperturbed Hamiltonian at a critical point and analyze the gradient $\nabla H(\s)$ and $\nabla^2 H(\s)$ separately. We can show that, for a spherical spin glass, the gradient at a critical point is zero-mean Gaussian with variance $n p$ while the Hessian is simply a scaled GOE matrix (cf. Lem.~\ref{lem:hessian_rmt} in Appendix~\ref{s:app:quality_perturbed_local_minima}). This helps us bound the location and the energy of the perturbed minimum as shown in the following theorem; a more formal version is Thm.~\ref{thm:bound_Y_s_delta_s} in Appendix~\ref{s:app:quality_perturbed_local_minima}.

\begin{theorem}
For large $n$ and a constant $\nu$, every local minimum $\s \in \crt_0(H)$ has another local minimum $\widetilde{\s} \in \crt_0(\tH)$ within a distance $\OO(n^{-1/3})$. Moreover, the normalized Hamiltonian differs by at most $2 \nu$ at these locations, i.e.,
\aeqs{
    \norm{\s - \widetilde{\s}}_2 &= \OO(n^{-1/3})\\
    \f{1}{n}\ \abs{H(\s) - \tH(\tilde{s})} &\leq 2 \nu.
}
\end{theorem}
In other words, the polynomial regime does not cause the local minima of the Hamiltonian to shift. Roughly, it only smears out small clusters of nearby local minima.

\section{Experiments}
\label{s:experiments}

In this section, we first empirically compute the energy landscape of a spin glass Hamiltonian and show how it changes with different levels of external perturbation. This gives a novel insight into structure of saddle points in the high-dimensional parameter space of a deep network. We then demonstrate our annealing procedure described in Sec.~\ref{ss:annealing} on fully-connected and convolutional neural networks for image classification on two datasets, viz., the MNIST dataset~\citep{lecun1998gradient} with $70,000$ $28\times28$ gray-scale images of digits, and the CIFAR-10 dataset~\citep{krizhevsky2009learning} with $60,000$ $32\times32$ RGB images of $10$ classes of objects. Even for large CNNs that are close to the state-of-the-art on CIFAR-10, we show that AnnealSGD results in improved generalization.

\subsection{Energy landscape of spin glasses}
\label{ss:expt_spin_glass}

Consider the perturbed 3-spin glass Hamiltonian with $n=100$ spins, in our notation this is a sparsely connected neural network with $100$ neurons on each layer, three hidden layers, threshold non-linearities and a fourth loss layer:
\beq{
    \label{eq:3_spin_perturbued}
    -\tH_{n,3} = \f{1}{n}\ \sum_{i,j,k=1}^{n}\ J_{i,j,k}\ \s_i \s_j \s_k + \sum_i\ h_i\ \s_i;
}
where $h_i \in N(0, \nu^2)$ and $J_{i,j,k} \sim N(0,1)$.
For a given a disorder $J$ we empirically find the local minima of the above Hamiltonian for values of $h$ in the three regimes:
\begin{enumerate}[(i)]
\item $\nu = 1/n$ for exponentially many critical points;
\item $\nu = \sqrt{3} (1 - 1/n)^{1/2}$ to be in the polynomial regime;
\item and a large value $\nu = 3$ to be in the totally trivialized regime with a single minimum.
\end{enumerate}

For an initial configuration $\s^1$ that is sampled uniformly on $S^{n-1}(\sqrt{n})$, we perform gradient descent with a fixed step size $\eta = 0.1$ and stop when $\norm{\nabla \tH}_2 \leq 10^{-4}$. The gradient can be computed to be
$$
    -\rbrac{\nabla \tH_{n,3}}_i = \f{1}{n}\ \sum_{j,k}\ \sqbrac{J_{ijk} \s_j \s_k + J_{jik} \s_j \s_k + J_{jki} \s_j \s_k} + h_i.
$$
Before each weight update, we project the gradient on the tangent plane of $S^{n-1}(\sqrt{n})$ and re-project {$\s^{k+1} \leftarrow \s^k - \eta \nabla \tH(\s^k)$} back on to the sphere.

We next perform a t-distributed stochastic neighbor embedding (t-SNE)~\citep{van2008visualizing} for $20,000$ runs of gradient descent from independent initial configurations. t-SNE is a nonlinear dimensionality reduction technique that is capable of capturing the local structure of the high-dimensional data while also revealing global structure such as the presence of clusters at several scales. This has been successfully used to visualize the feature spaces of deep networks~\citep{donahue2014decaf}. We use it to visualize the high-dimensional energy landscape.

The results of this simulation are shown in Fig.~\ref{fig:tsne}. Local minima are seen as isolated points with surrounding energy barriers in Fig.~\ref{fig:p3t0}. The saddle points in this t-SNE of the $100$-dimensional hyper-sphere are particularly interesting. They are seen as long ``canyons'', i.e., they are regions where the gradient is non-negative in all but a few directions, gradient descent is only free to proceed in the latter. As multiple independent runs progress along this canyon, we get the narrow connected regions in Fig.~\ref{fig:p3t0} and Fig.~\ref{fig:p3t1}.

\begin{figure*}[!tbh]
\centering
    \begin{subfigure}[t]{0.32\textwidth}
        \centering
        \includegraphics[width= 0.95 \textwidth]{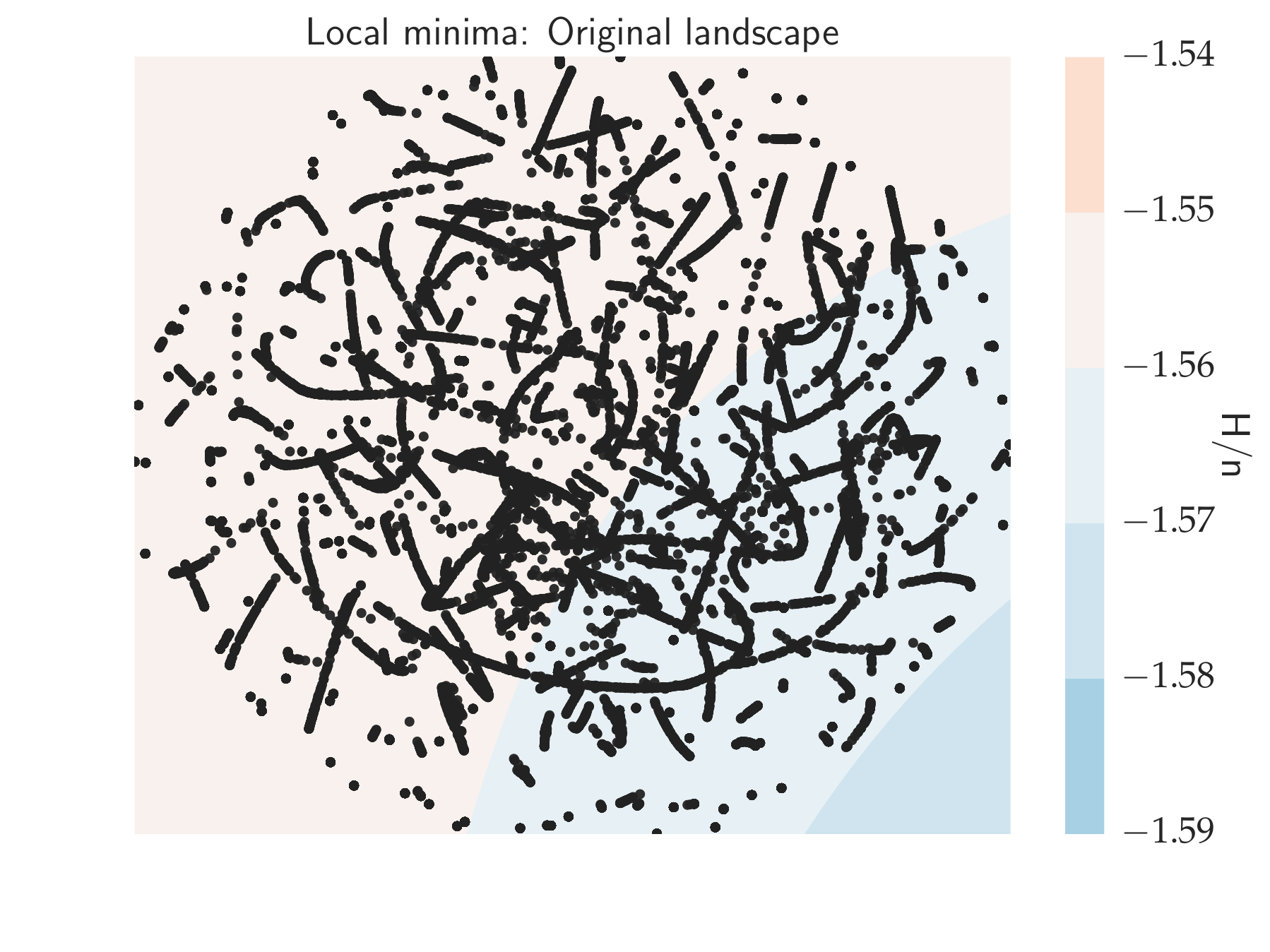}
        \caption{\small \linespread{1.05} Exponential regime: Local minima are seen here as isolated dots surrounded by high energy barriers while saddle points are seen as narrow connected regions, viz., regions where the gradient is very small in all but a few directions.}
        \label{fig:p3t0}
    \end{subfigure}
    \hfill
    \begin{subfigure}[t]{0.32\textwidth}
        \centering
        \includegraphics[width= 0.95 \textwidth]{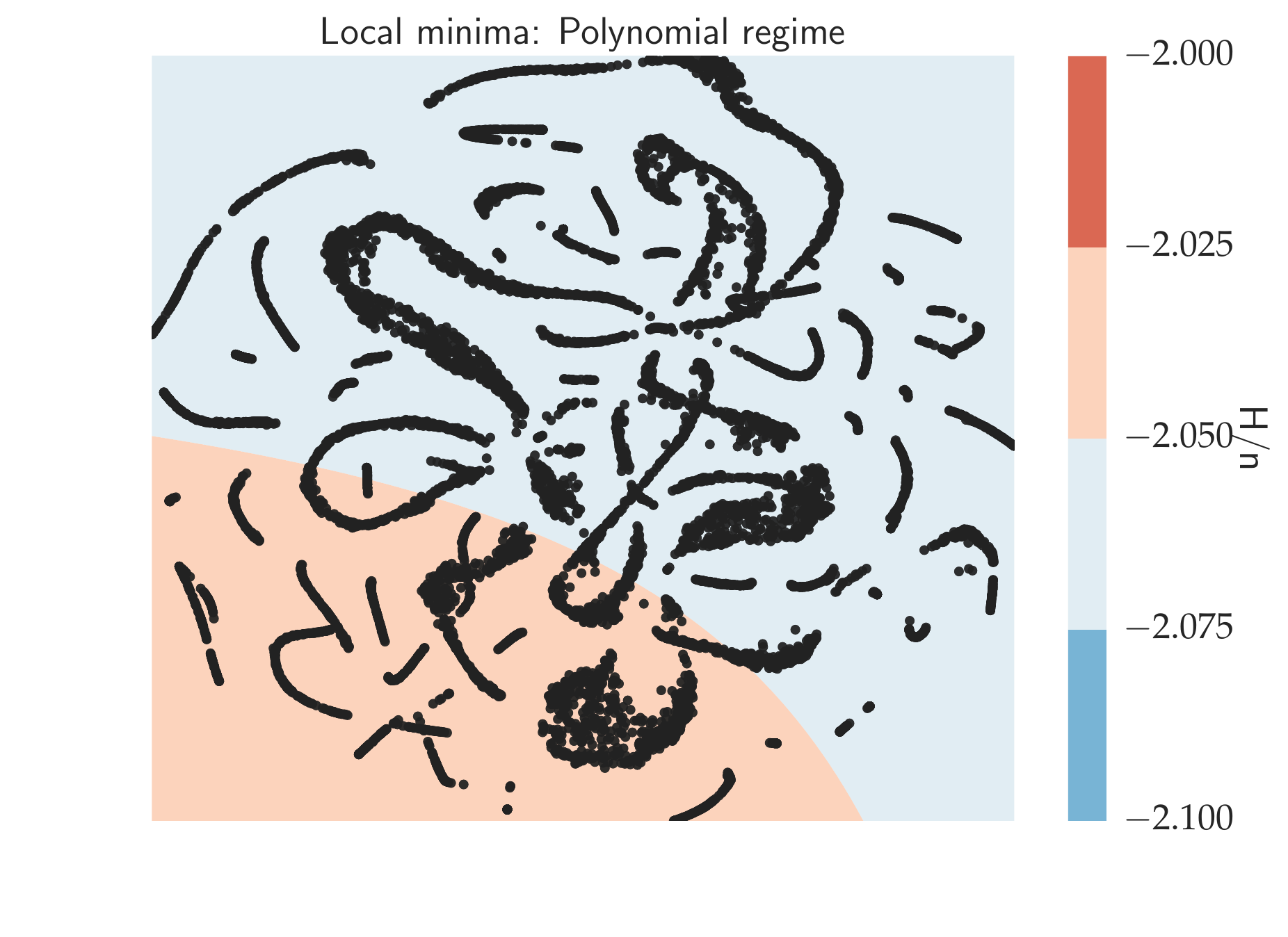}
        \caption{\small \linespread{1.05} Polynomial regime: The number of isolated clusters, i.e., local minima, is significantly smaller as compared to Fig.~\ref{fig:p3t0}. As the discussion in Sec.~\ref{ss:scaling_local_minima} predicts, the energy landscape seems to be full of saddle points in the polynomial regime.}
        \label{fig:p3t1}
    \end{subfigure}
    \hfill
    \begin{subfigure}[t]{0.32\textwidth}
        \centering
        \includegraphics[width= 0.95 \textwidth]{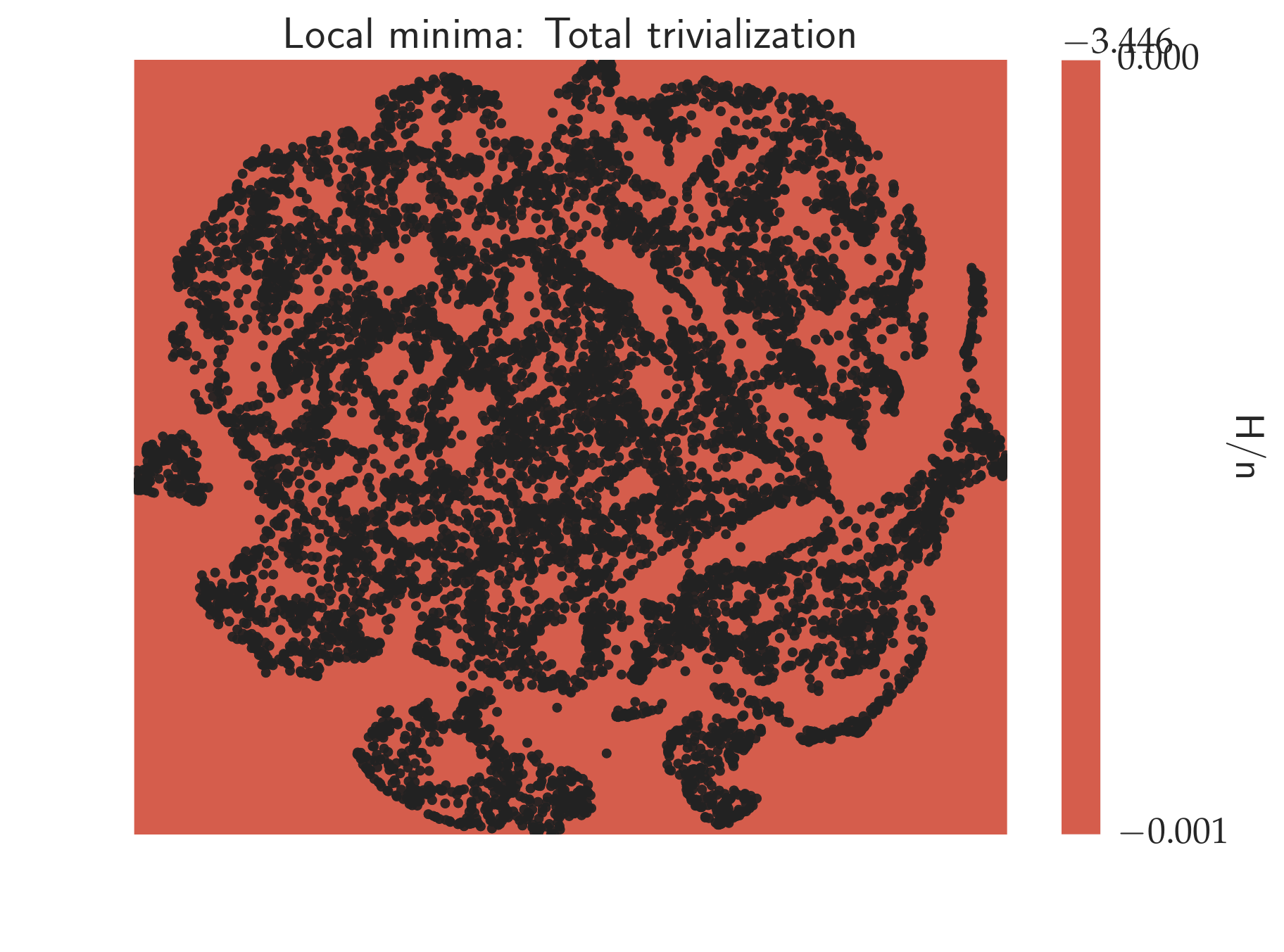}
        \caption{\small \linespread{1.05} Trivial regime: Gradient descent always converges to the same location. The average cosine distance (on $S^{n-1}(\sqrt{n})$) here is $0.02$ as compared to $1.16$ for Fig.~\ref{fig:p3t0} which suggests that this is indeed a unique local minimum.}
        \label{fig:p3t2}
    \end{subfigure}
\caption{Two-dimensional t-SNE~\citep{van2008visualizing} of $20,000$ of local minima discovered by gradient descent for the three regimes: exponential (Fig.~\ref{fig:p3t0}), polynomial (Fig.~\ref{fig:p3t1}) and trivial (Fig.~\ref{fig:p3t2}). The background is colored by the kernel density estimate of the value of the normalized Hamiltonian $H/n$. The numerical values of the normalized Hamiltonian in Fig.~\ref{fig:p3t0} are the non-asymptotic versions ($n=100$) of the values in Fig.~\ref{fig:energy_barriers}.}
\label{fig:tsne}
\end{figure*}

\subsection{Setup for deep neural networks}
\label{ss:expt_dnn}

The theoretical model in Sec.~\ref{ss:model_deep_networks} and the annealing scheme~\eqref{eq:nu_polynomial} depend on two main parameters: the number of neurons on each layer $n$ and the number of hidden layers $p$. Modern deep networks with multiple feature maps, convolutional kernels etc. do not conform to our theoretical architecture and we therefore use a heuristic to set these parameters during training. For all networks considered here, we set
$$
    n := \rbrac{\f{\#\ \trm{weights}}{p}}^{1/2}
$$
where $p$ is taken to be the number of layers. This estimate for the number of neurons is accurate for a fully-connected, dense neural network of $p$ layers. For CNNs, we consider the combined block of a convolution operation followed by batch normalization and max-pooling as a single layer of the theoretical model. All linear and convolutional layers in what follows have rectified linear units (ReLUs) non-linearities and we use the cross-entropy loss to train all our networks.

Thm.~\ref{thm:transition_critical_dot} depends on the spherical constraint introduced in Sec.~\ref{ss:model_deep_networks} quite crucially. Indeed, the fact that the spherical spin glass Hamiltonian is isotropic on the hyper-sphere is essential for us to be able to pick the perturbation $h$ in a ``random'' direction. Modern deep networks are not trained with additional constraints on weights, they instead use $\ell_2$ regularization (weight decay) which can be seen as a soft $\ell_2$ constraint on the weights. This case can also be analyzed theoretically and the annealing strategy only undergoes minor variations~\citep{fyodorov2013high} (the polynomial regime is achieved for $\abs{B} = \OO(n^{-1/2}$). We however use the annealing scheme in Sec.~\ref{ss:annealing} with $\abs{B} = \OO(n^{-1})$ here.

For MNIST, we scale each pixel to lie between $[-1,1]$. For CIFAR-10, we perform global contrast normalization and ZCA-whitening~\citep{goodfellow2013maxout} as the preprocessing steps. We do not perform any data augmentation. For both datasets, we use a subset of size $12\%$ of the total training data as a validation set and use it for hyper-parameter tuning of the standard pipeline without gradient perturbation. For the best parameters, we run both the standard optimizer and AnnealSGD with the same random seed, i.e., both networks have the same weight initialization, batch-sampling and the same hyper-parameters, and report the error rates on the test set averaged over a few runs.

\subsection{Fully-connected network}
\label{ss:expt:fc}

As a sanity check, we construct a thin, very-deep fully-connected network to demonstrate topology trivialization. Deep networks with a large number of connections, viz., fully-connected ones, are often challenging to train due to poorly initialized weights and vanishing gradients~\citep{sutskever2013importance}. The external magnetic field results in a gradient bias and thus artificially boosts the magnitude of the gradient which should effectively counter this problem.

Our ``$\mnistfc$'' network has $16$ hidden layers with $128$ hidden units and we use a momentum-based optimizer ADAM~\citep{kingma2014adam} with learning rate $10^{-3}$ for $50$ epochs to train it. We set $J = 10^{-3}$ and $\tau_0 = 500$ along with an $\ell_2$ regularization of $10^{-5}$. As Fig.~\ref{fig:mnistfc_test} shows, averaged over $10$ runs, AnnealSGD results in dramatically faster training in the beginning when the external field is stronger. Moreover, as predicted, the minimum of the absolute value of the back-propagated gradient with AnnealSGD in Fig.~\ref{fig:mnistfc_grad} is significantly larger than the original gradient. The final validation error for annealing is also slightly smaller ($2.48\%$) as compared to the error without it ($2.92\%$).

\begin{figure}[!tbh]
\centering
    \begin{subfigure}[b]{0.32\columnwidth}
        \centering
        \includegraphics[width=0.96\textwidth]{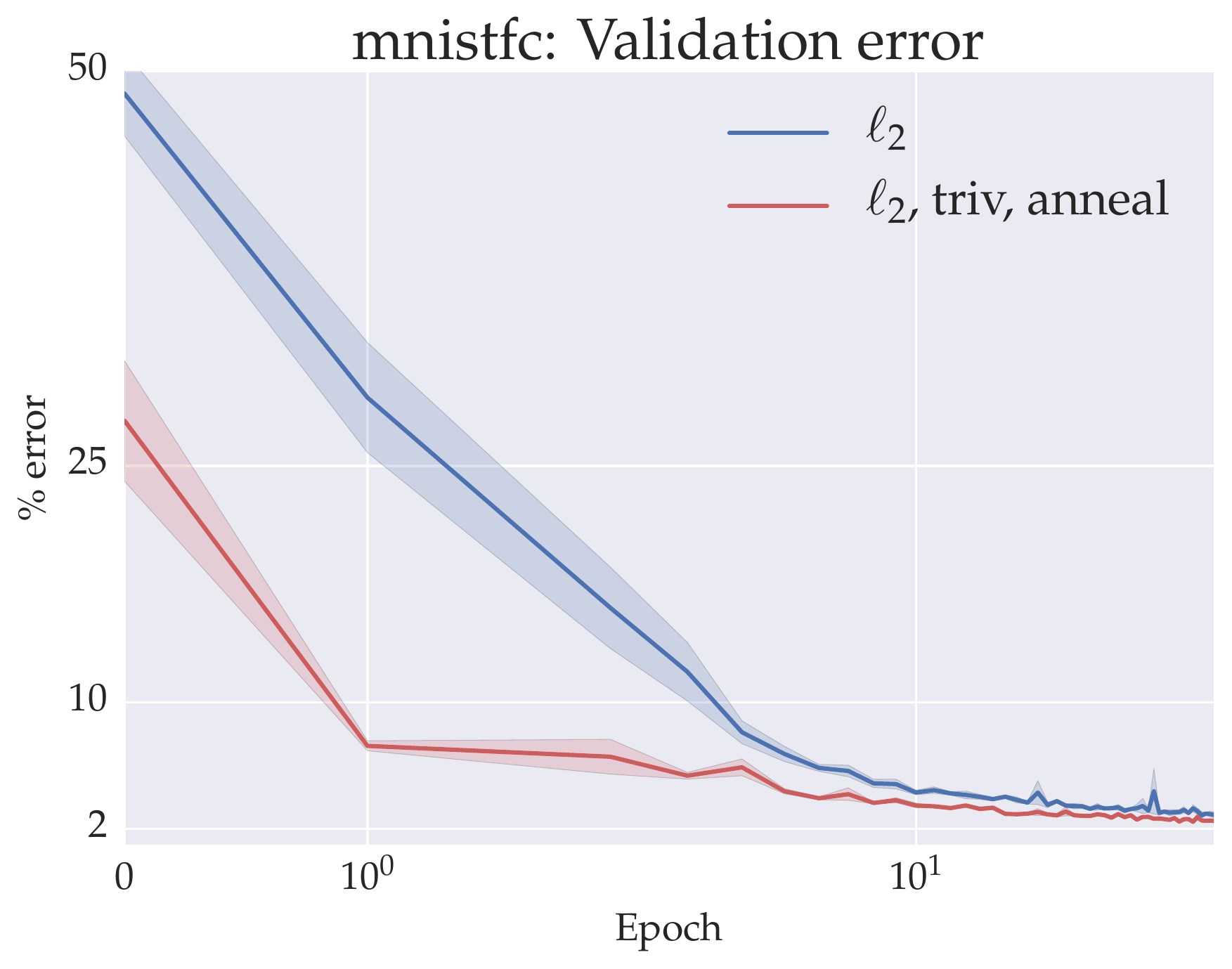}
        \caption{}
        \label{fig:mnistfc_test}
    \end{subfigure}
    \hfill
    \begin{subfigure}[b]{0.32\columnwidth}
        \centering
        \includegraphics[width=\textwidth]{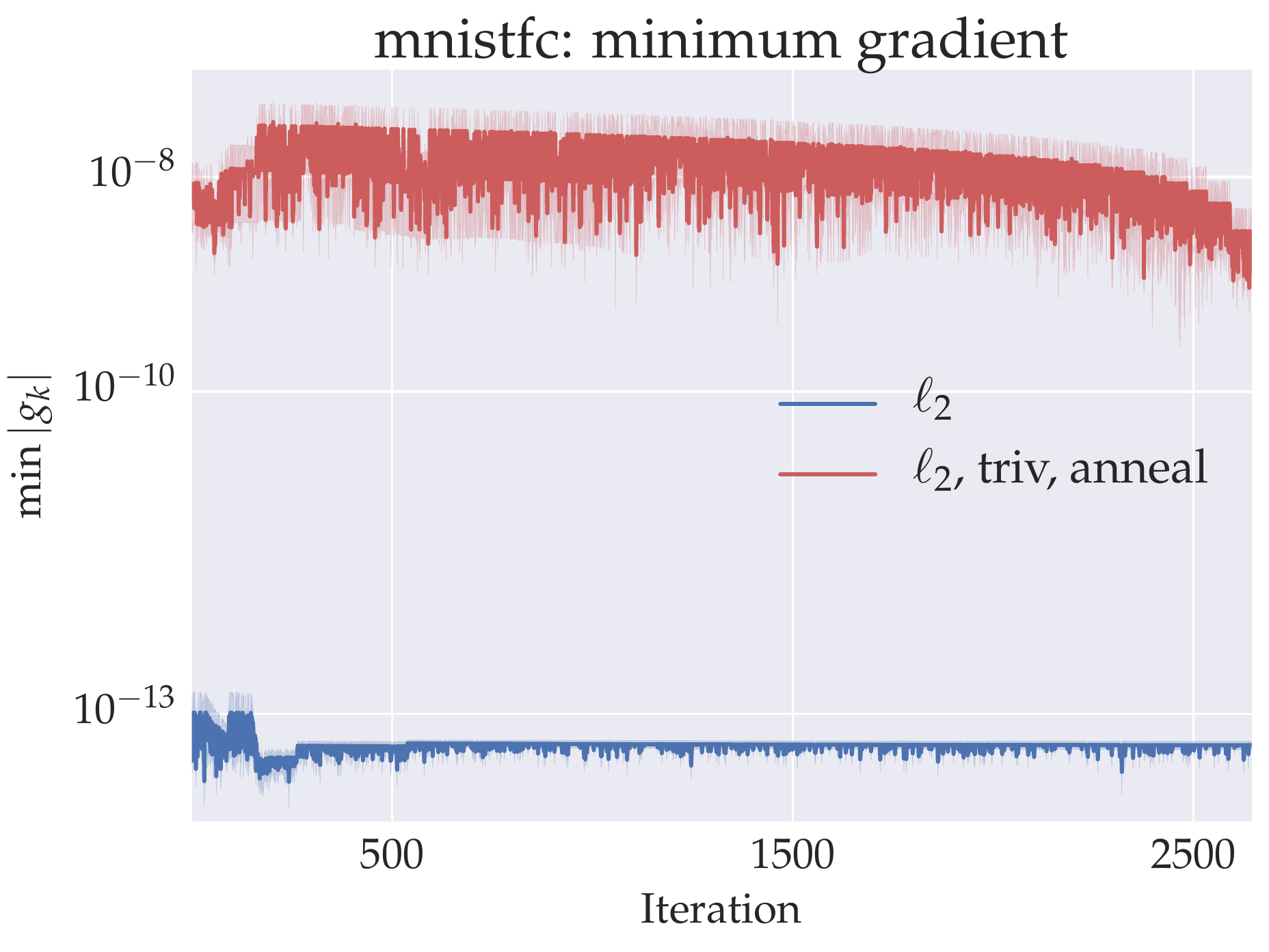}
        \caption{}
        \label{fig:mnistfc_grad}
    \end{subfigure}
    \hfill
    \begin{subfigure}[b]{0.32\columnwidth}
        \centering
        \includegraphics[width=0.96\textwidth]{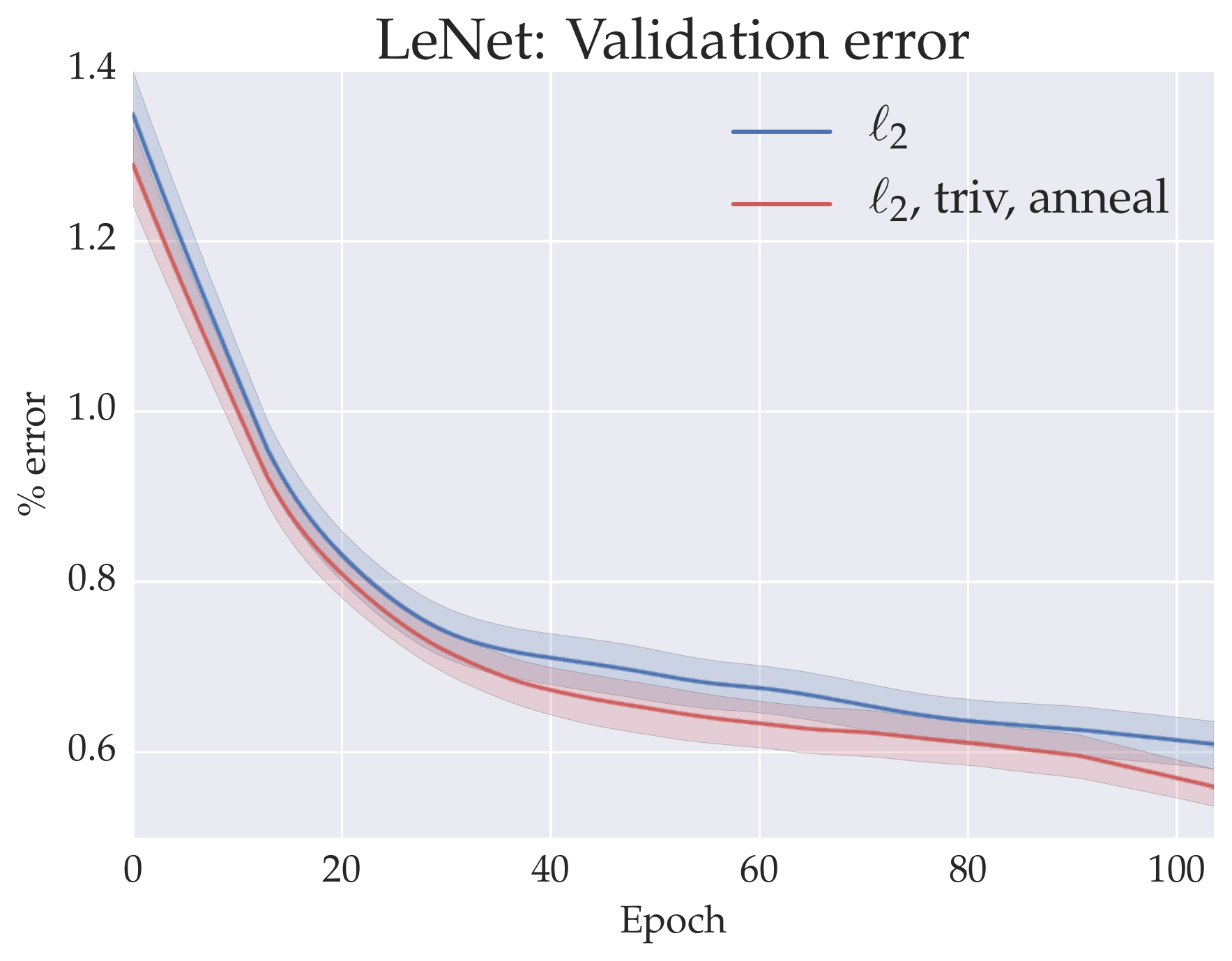}
        \caption{}
        \label{fig:mnistconv}
    \end{subfigure}
\caption{Fig.~\ref{fig:mnistfc_test}: $\mnistfc$ trained with ADAM (blue) vs. AnnealSGD (red). Fig.~\ref{fig:mnistfc_grad}: Minimum absolute value of the back-propagated gradient during training for ADAM (blue) vs. AnnealSGD (red). Fig.~\ref{fig:mnistconv} shows the validation error for LeNet trained using ADAM (blue) vs. AnnealSGD (red).}
\label{fig:mnist}
\end{figure}

\subsection{LeNet on MNIST}
\label{ss:expt:lenet}

We now train a LeNet-style network~\citep{lecun1998gradient} on $28\times28$ gray-scale images with the following architecture:
\aeqs{
    \block_{32} \to \block_{64} \to \linear_{512} \to \dropout_{\ 0.1} \to \softmax.
}
The $\block_{32}$ layer denotes a convolutional layer with a $5\times5$ kernel and $32$ features, $3\times 3$ max-pooling and batch-normalization. Note that as compared to the standard LeNet, we have replaced the $\dropout$ layer after the convolutional layer with batch-normalization; this is common in recent papers~\citep{ioffe2015batch}. We train using ADAM with a learning rate $\eta = 10^{-4}$ for $100$ epochs with a batch size of $256$ and $\ell_2$ regularization of $10^{-3}$ and compare it with AnnealSGD using $J = 10^{-4}$ and $\tau_0 = 10^3$.

Fig.~\ref{fig:mnistconv} shows the average validation error over $10$ runs. Typically, convolutional networks are very well-regularized and the speedup here in training is smaller than Fig.~\ref{fig:mnist}. However, we still have an improvement in generalization error: $0.55\ \pm\ 0.03 \%$ final validation error with annealing compared to $0.61\ \pm\ 0.04 \%$ without it.

\subsection{All-CNN-C on CIFAR-10}
\label{ss:expt:allcnn}

As a baseline for the CIFAR-10 dataset, we consider a slight modification of the All-CNN-C network of~\citet{springenberg2014striving}. This was the state-of-the-art until recently and consists of a very simple, convolutional structure with carefully chosen strides:
\aeqs{
    \data_{3\times32\times32}& \to \dropout_{\ 0.2} \to \big(\block_{96} \big)_{\times 3} \to \dropout_{\ 0.5} \to \big(\block_{192} \big)_{\times 3}\\
    \ldots & \to \dropout_{\ 0.5} \to  \big(\block_{192} \big)_{\times 2} \to \meanpool_{8\times 8} \to \softmax.
}
The $\block_{96}$ layer denotes a $\convolution_{3\times3\times96}$ layer followed by batch normalization (which is absent in the original All-CNN-C network). Unfortunately, we could not get ADAM to prevent over-fitting on this network and hence we train with SGD for $200$ epochs with a batch-size of $128$ and Nesterov's momentum. We set momentum to $0.9$ and the initial learning rate to $0.1$ which diminishes by a factor of $5$ after every $60$ epochs. We also use an $\ell_2$ regularization of $10^{-3}$. The optimization strategy and hyper-parameters were chosen to obtain error rates comparable to~\citet{springenberg2014striving}; however as Table~\ref{tab:cifar10} shows, we obtain a $0.66\%$ improvement even without annealing on their results ($8.42\%$ test error in $200$ epochs vs.\ $9.08\%$ error in $350$ epochs). This can probably be attributed to batch normalization.

We compare the training of the baseline network against AnnealSGD with $J = 10^{-4}$ and $\tau_0 = 10^3$ as the hyper-parameters over two independent runs. Table~\ref{tab:cifar10} shows a comparison with other competitive networks in the literature. We note that All-CNN-BN with annealing results in a significantly lower test error ($7.95\%$) than the baseline ($8.42\%$).

{
\renewcommand{\arraystretch}{1.2}
\begin{table}[H]
\centering
\resizebox{0.65 \columnwidth}{!}
{%
\begin{tabular}{p{6.5cm} r}
    Model & Test error ($\%$)\\  \hline
    NiN~\citep{lincy13} & $10.41$\\
    All-CNN-C~\citep{springenberg2014striving} & $9.04$\\
    P4-All-CNN~\citep{cohen2016group} & $8.84$\\
    All-CNN-BN (baseline) & $8.42\ \pm\ 0.08$\\
    All-CNN-BN (resampled, annealed noise) & $8.16\ \pm\ 0.07$\\
    All-CNN-BN (with topology trivialization) & $7.95\ \pm\ 0.04$\\
    ELU~\citep{clevert2015fast} & $6.55$\\ \hline
\end{tabular}
}
\caption{\small Error rates on CIFAR-10 without data augmentation.}
\label{tab:cifar10}
\end{table}
}

\subsection{Comparison to gradient noise}

\begin{wrapfigure}{r}{0.5 \columnwidth}
\centering
\includegraphics[width=0.45 \columnwidth]{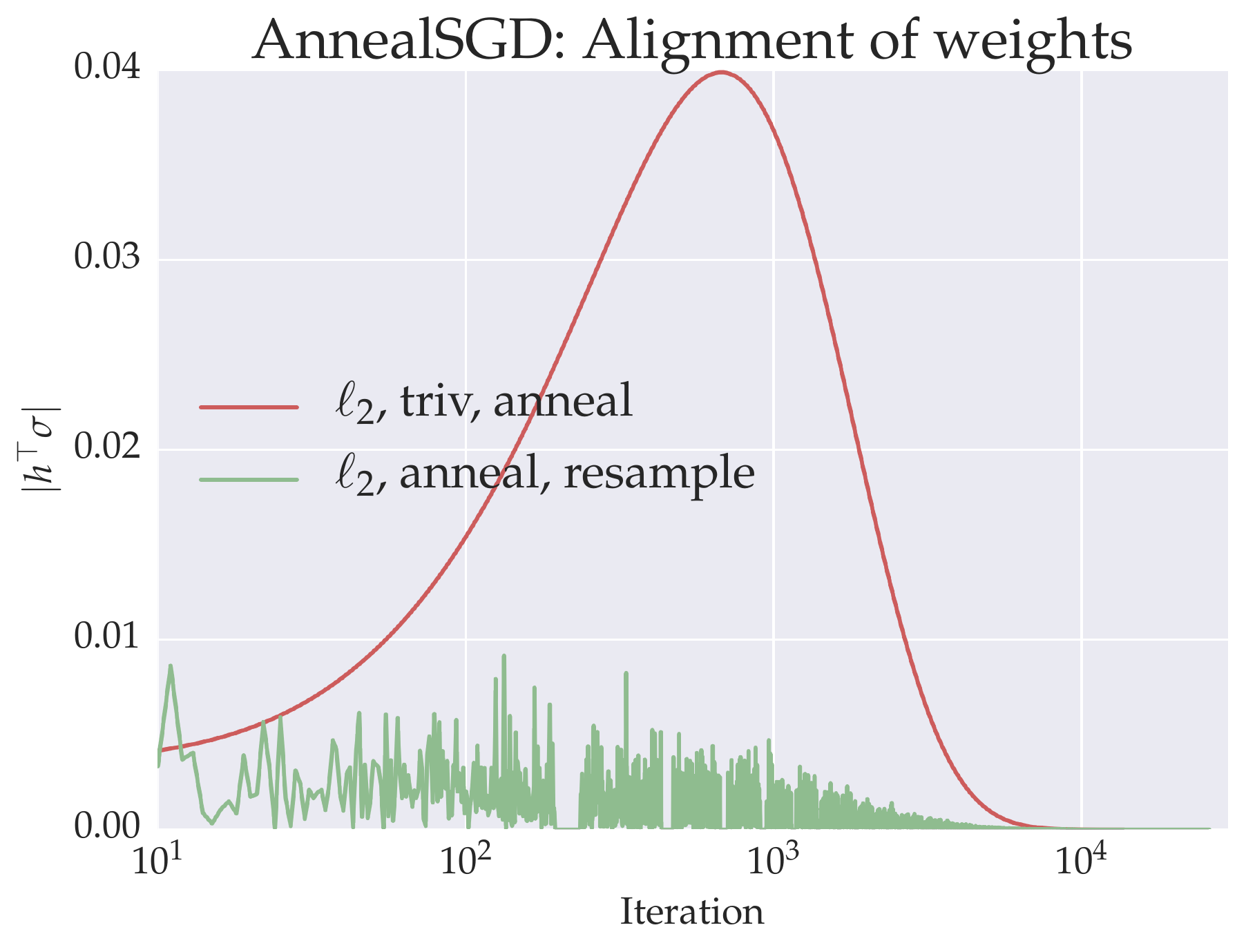}
\caption{\small Alignment of the weights with the perturbation term $\rbrac{\abs{h^\top \s}}$ for All-CNN-BN trained with AnnealSGD (red) vs. resampled, annealed additive noise (green) (cf. Table~\ref{tab:cifar10} for error).}
\vspace*{-0.5in}
\label{fig:cifarconv_align}
\end{wrapfigure}

Annealed additive gradient noise has also been recently explored for complex recurrent network architectures in~\citet{neelakantan2015adding}. We compare AnnealSGD with this approach and as Table~\ref{tab:cifar10} shows, although resampled noise improves slightly upon the baseline, fixing the direction of the additive term works better. This is in line with our analysis using an external magnetic field, as Fig.~\ref{fig:cifarconv_align} shows, resampling is detrimental because it forces the weights to align in different uncorrelated directions at each iteration whereas the annealing scheme simply introduces a (modulated) bias in a fixed direction.

\vspace*{0.4in}
\section{Discussion}
\label{s:discussion}

Our analysis was conducted for spin glasses and indeed, our random sparse model is very different from say, dense convolutional networks with Gabor filters. Nevertheless, a number of empirically observed phenomenon in the deep learning literature that can be rigorously proved for spin glasses connect the two, e.g., multiple equivalent local minima, low-order saddle points close to the ground state versus easy to escape high-order saddle points higher up etc. Insight on these properties can be exploited to design algorithms that can be tested empirically on modern deep networks, which is what we have done here. We certainly advocate further study into the properties of the landscape of convolutional architectures, for which one could use the analysis of spin glasses at least as a loose source of inspiration. This is a part of future work, well beyond the scope of this paper.

In view of the effects of an external magnetic field on the loss function of deep networks, it is an interesting question whether the inherent gradient noise of SGD has similar effects; this might inform a strategy for choosing batch sizes and data augmentation. Doing so however requires an analysis of the dynamics of gradient descent, instead of the ground-state thermodynamics of spin glasses which we have exploited here.

\section{Conclusions}
\label{sec:conclusions}

We have demonstrated ``AnnealSGD'', a regularized version of SGD for deep networks that uses an annealed additive gradient perturbation to control the complex energy landscape. We connected a random sparse model for deep networks to the Hamiltonian of a spin glass and identified three complexity regimes with exponential, polynomial and a constant number of local minima. AnnealSGD modulates the gradient perturbation to transition across these regimes. Our experiments show that this technique, which is very easy to implement, improves the upon the training of modern convolutional networks.

{
\footnotesize
\bibliography{chaudhari.soatto.aistats17}
\bibliographystyle{apalike}
}

\clearpage
\begin{appendices}

\setcounter{equation}{0}
\setcounter{theorem}{0}
\setcounter{figure}{0}
\makeatletter
\renewcommand{\thefigure}{A-\arabic{figure}}
\renewcommand{\thetheorem}{A-\arabic{theorem}}
\renewcommand{\theequation}{A-\arabic{equation}}

\renewcommand\thetable{\thesection\arabic{table}}
\renewcommand\thefigure{\thesection\arabic{figure}}

\section{Proof of Theorem~\ref{thm:deep_net_H}}
\label{s:proofs}

Consider the model of a deep network discussed in Sec.~\ref{ss:model_deep_networks} that is given by
\beq{
    Y = g \rbrac{\W{p+1} g \rbrac{ \W{p} \ldots\ g \rbrac{\W{1} X - \f{d}{3} \ones_n} - \f{d}{3} \ones_n} \ldots };
    \label{app:eq:X_to_Y}
}
where $g(x) = \ind_{ \cbrac{x \geq 0}}$ is the thresholding function. In order to get rid off the non-linearities, we can write~\eqref{app:eq:X_to_Y} in terms of its ``active paths'', i.e., all paths with non-zero activations from an input neuron $X_i$ to the output $Y$. We then have
\beq{
    \label{eq:Y_path_form}
    Y = \sum_{i=1}^n\ \sum_{\g \in \Gamma_i}\ X_i\ W_\g;
}
where $\Gamma_i$ is the set of all active paths connecting the $i^{th}$ input neuron to the output $Y$. Let $W_\g$ be the product of the weights along each path $\g$. Now define $\g_{i, i_1, \ldots, i_p}$ to be an indicator random variable that is 1 iff the path $i \to i_1 \to \ldots \to i_p \to Y$, i.e., the one that connects the $i_1^{th}$ neuron in the first hidden layer, $i_2^{th}$ neuron in the second hidden layer and so on, is active. With some abuse of notation, we will use the variable $\g$ to denote both the path and the tuple $(i, i_1, i_2, \ldots, i_p)$.

With an aim to analyze $Y$ for a generic data distribution $X$, we now work towards removing the dependency of the input $X$ in~\eqref{eq:Y_path_form}. First split the path into two parts, the first edge on layer 1 and the rest of the path as: $\g_{i, i_1, i_2, \ldots, i_p} = \g_{i, i_1}\ \g_{i_1, \ldots, i_p}$ (conditional upon $\g_{i,i_1} = 1$, the sub-paths are independent). This gives
\aeqs{
    Y = \sum_{i,\ip =1}^n\ \g_{i, i_1}\ \gip\ X_i\ W_{i,i_1}\ \Wip.
    \label{eq:Y_original}
}
This expression has correlations between different paths that are introduced by the sum over inputs $X_i$ which we approximate next. Define the net coupling due to the sum over the inputs as
$$
    Z_i := X_i\ \g_{i, i_1}\ W_{i, i_1}, \quad Z := \sum_{i=1}^n\ Z_i.
$$
Since $X_i \in \zoo$, we have $\var\ (X_i) \leq 1/4$ (the variance is maximized for a Bernoulli random variable with probability $1/2$), and according to assumption 2 in Sec.~\ref{ss:model_deep_networks}, $W_{i,i+1}$ is positive or negative with equal probability $d/2n$ while $\g_{i,i_1} =1$ with probability $d/n$ for a fixed $i_1$. This gives
$$
    \E\ Z_i = 0 \quad \trm{and}\ \quad \E\ (Z_i^2) \leq d^2/(4 n^2).
$$
We now use Bernstein's inequality to see that:
\aeqs{
    P \rbrac{ \abs{\sum_{i=1}^n\ Z_i} > t} &\leq \exp \rbrac{ -\f{t^2/2}{\var\ (Z) + t/3} }\\
    \implies \abs{Z} &\leq n^{-\f{1}{2} + \f{1}{p} + \e} \leq n^{-1/5} \qquad \trm{for}\ p > 5;
}
with high probability for any small $\e > 0$ (we picked $\e = 1/5$ above). This also gives $\var\ (Z) \leq \f{n^{-2/5}}{4}$. We now get a lower bound on $\var\ (Z)$ using Chebyshev's inequality where we use the assumption that the weight distribution is not all concentrated around zero. In particular, let us assume that $\P \rbrac{\abs{\W{k}_{ij}} > t} \geq t$ for some small $t = n^{-2/15}$; note that the special form $\P(W > t) \geq t$ is only for convenience. This gives $\var\ (W) \geq t^3 = n^{-2/5}$ and finally
\beq{
    \f{n^{-2/5}}{4} \geq \var\ (Z) \geq \E\ (X^2)\ n^{-2/5 + 1/p}.
    \label{eq:variance_Z}
}
We have thus shown that the net coupling due to inputs in our model is bounded and has variance of the order $\Th(n^{-2/5})$.

We now compute the probability of a path $\g_{i_1, \ldots, i_p}$ being active. Note that hidden units on the $k^{th}$ layer are uniform, but because of correlations of the layers above, they are not simply Bernoulli with probability $\f{\r}{n} \rbrac{\f{d}{2}}^{p-k}$. However, for networks that are at most $p = \OO(\log_d n)$ deep, we can bound the probability of the $\ell^{th}$ neuron being active on the $k^{th}$ layer~\citep[Proof of Lem.~6]{arora2013provable}:
$$
    \P \rbrac{\h{k}_\ell = 1} \in \sqbrac{ \f{3 \r}{4 n}\ \rbrac{\f{d}{2}}^{p-k}, \f{5 \r}{4 n}\ \rbrac{\f{d}{2}}^{p-k} };
$$
which gives $P \rbrac{\g_{i_1, \ldots, i_p} = 1} = \Th(n^{-(p-1)/2})$. The random variable $Z\ \gip$ is zero mean with variance of the order $\Th(n^{-(p-1)/2 -2/5})$ and we therefore replace it with a zero-mean standard Gaussian $\Kip$
\beq{
    Y_J \overset{law}{=} \f{J}{n^{(p-1)/4 + 1/5}}\ \sum_{\ip=1}^n\ \Kip \Wip,
    \label{eq:replace_J}
}
for some constant $J > 0$.

Since the entries of weight matrices $\W{k}$ are iid, the weights along each path $\g \in \Gamma_i$, i.e., $\Wip$ are independent and identically distributed. Note that there are at most $\OO(m^p)$ distinct values that $\Wip$ can take since each $W_{i_k, i_{k+1}}$ is supported over a set of cardinality $m = \Th(n)$ in the interval $[-N,\ N]$. For $\s_{l_k} \in \trm{supp}(W_{i_k, i_{k+1}})$ with $l_k \leq n$, we now write $\Wip$ as a sum of weighted indicator variables to get
\aeqs{
    \Wip = \sum_{l_1, \ldots, l_p = 1}^m\ \ind\cbrac{W_{i_1,i_2} = \s_{l_1}, \ldots, W_{i_p,Y} = \s_{l_p}}\ \rbrac{\prod_{k=1}^p\ \s_{l_k}},
}
to get
\aeqs{
    Y_J &= \f{J}{n^{(p-1)/4 + 1/5}}\ \sum_{l_1, \ldots, l_p = 1}^m\ \sum_{\ip=1}^n\ \Kip \ind\cbrac{W_{i_1,i_2} = \s_{l_1}, \ldots, W_{i_p,Y} = \s_{l_p}}\ \rbrac{\prod_{k=1}^p\ \s_{l_k}},\\
    &\triangleq \f{J}{n^{(p-1)/4  + 1/5}}\ \sum_{l_1, \ldots, l_p = 1}^m\ \Jlp \rbrac{\prod_{k=1}^p\ \s_{l_k}};
}
Note that $\Jlp$ is a zero-mean Gaussian with variance $(n/m)^p$ and we can write the scaled output $\hY = \f{n^{-(p-3)/4 + 1/5}}{\sqrt{m}}\ Y_J$ after a renaming of indices as
\beq{
    \hY \overset{\trm{law}}{=} \f{J}{m^{(p-1)/2}}\ \sum_{\ip = 1}^m\ \Jip \ \sip.
    \label{eq:Yhat}
}
Since each $\s_i$ has the same distribution as $\W{k}_{ij}$, the $p$-product $\sip$ can take at most $\binom{n}{p}$ distinct values for every configuration $\s$. Using Stirling's approximation, we can see that this is also of the order $\OO(m^p)$ for $p = \OO(\log n)$ and $m = \Th(n)$, i.e., if the discretization of weights grows with the number of neurons we do not lose any representational power by replacing the weights by spins. In the sequel, we will simply set $m = n$ the formulas undergo only minor changes if $m = \Th(n)$.

Now, the zero-one loss for target labels $Y^t \sim \Berdist(q)$ can be written as
\aeqs{
    L = \E_{Y^t} \abs{\hY(X^t) - Y^t} &= q(1-\hY) + (1-q)\ \hY\\
    &= q + (1-2q)\ \hY.
}
Define $H_{n,p} = (q-L)/(1-2q)$ to see that $-H_{n,p}$ has the same distribution as $\hY$. In Thm.~\ref{thm:deep_net_H} in the main writeup, we assimilate the factor $1-2q$ into the constant $J$.

\section{Quality of perturbed local minima}
\label{s:app:quality_perturbed_local_minima}

The main result of this section is Thm.~\ref{thm:bound_Y_s_delta_s} which says that if the external field is such that if $\nu$ is a constant, the perturbations of the normalized Hamiltonian are $\nu$ in the limit while the perturbations of local minima are at most $\OO(n^{-1/3})$. The development here is based on some lemmas in~\citet{subag2015extremal} which undergo minor modifications for our problem. The basic idea consists of approximating the perturbed Hamiltonian as a quadratic around a local minimum of the original Hamiltonian and analyzing the corresponding local minimum.

Let us consider the Hamiltonian
\beq{
    -\tH(\s) = -H(\s) + \f{1}{\sqrt{n}}\ h^\top \s;
    \label{eq:tH_diff_normalization}
}
where $h \in \reals^n$ is a vector with zero-mean Gaussian iid entries with variance $\nu^2$ and
$$
    -H(\s) = n^{-(p-1)/2} \sum_{\ip} \Jip\ \sip.
$$
Note that this is the same model considered in~\eqref{eq:tH_general} except that we have used a different normalization for the perturbation term. We first get rid off the spherical constraint by re-parameterizing the Hamiltonian around a local minimum $\s \in \crt_0(H)$. For $x = (x_1, \ldots, x_{n-1}) \in \reals^{n-1}$ with $\norm{x} \leq 1$ and any $y \in \reals^n$ consider the functions
\aeqs{
    P_E(x) &= \rbrac{x_1, x_2, \ldots, x_{n-1}, \sqrt{1 - \norm{x}^2}},\\
    S(y) &= \sqrt{n}\ y.
}
Now define
\beq{
    f_\s(x) \triangleq -\f{1}{\sqrt{n}}\ H\ \circ\ \th_\s\ \circ\ S\ \circ\ P_E(x);
}
where $\s \in S^{n-1}(1)$ is the normalized spin on the unit-sphere and $\th_\s$ is a rotation that sends the the north pole to $\s$. The function $f_\s(x)$ is thus the re-parameterization of $H(\s)$ around a local minimum $\s$, but normalized to have a unit variance. We define $\tf_\s(x)$ similarly and also define $\lf_\s(x)$ for the perturbation term $h^\top \s$ to see that~\eqref{eq:tH_diff_normalization} becomes
$$
    \sqrt{n}\ \tf_\s(x) = \sqrt{n} f_\s(s) + \lf_\s(x).
$$
We can now construct the quadratic approximation of $\tf_\s(x)$ (which we assume to be non-degenerate at local minima, i.e, the Hessian has full rank) to be
\aeqs{
    \sqrt{n}\ \tf_{\s, \trm{approx}}\ (x) \triangleq \sqrt{n}\ f_{\s, \trm{quad}}(x) + \lf_{\s, \trm{lin}}(x);
}
where we used the quadratic approximation of the original Hamiltonian and the linear approximation of the perturbation term:
\aeq{
    f_{\s, \trm{quad}}(x) &= f_\s(0) + \ag{\nabla f_\s(0), x} + \f{1}{2} \ag{\nabla^2 f_\s(0) x,\ x}, \notag\\
    \lf_{\s, \trm{lin}}(x) &= \lf_\s(0) + \ag{\nabla \lf_\s(0), x}
    \label{eq:quad_lin_approx}.
}
Let $Y_\s$ be the critical point of $\sqrt{n}\ \tf_{\s, \trm{approx}}$ and $V_\s$ be its value. We have,
\aeq{
    Y_\s &\triangleq -\f{1}{\sqrt{n}}\ \rbrac{\nabla^2 f_\s(0)}^{-1}\ \nabla \lf_\s(0);  \label{eq:Y_s}\\
    \Delta_\s &\triangleq \sqrt{n}\ \tf_\s(Y_\s) - \sqrt{n}\ f_\s(0) - \lf_\s(0), \notag\\
    &= -\f{1}{2 \sqrt{n}}\ \rbrac{\nabla \lf_\s(0)}^\top \rbrac{\nabla^2 f_\s(0)}^{-1}\ \nabla \lf_\s(0); \label{eq:delta_s}
}
where $Y_\s$ is the critical point of $\tf_\s$. The term $\sqrt{n} f_\s(0)$ is the value of the original Hamiltonian at the unperturbed minimum and thus $\D_\s$ is the perturbation in the critical value of the Hamiltonian due to the perturbations minus the contribution of the perturbation $\lf_\s(0)$ itself. We now allude to the following lemma that gives the distribution of the gradient and the Hessian of the Hamiltonian at the critical point, viz., $\nabla f_\s(0)$ and $\nabla^2 f_\s(0)$.

\begin{lemma}[\citet{auffinger2013random}, Lem.~3.2]
\label{lem:hessian_rmt}
The gradient $\nabla f_\s(0)$ is zero-mean Gaussian with variance $p$ and is independent of $\nabla^2 f_\s(0)$ and $f_\s(0)$. Also, conditional on $f_\s(0) = u$, the Hessian $\nabla^2 f_\s(0)$ has the distribution
\beq{
    \sqrt{2 (n-1) p (p-1)} M - p u I;
    \label{eq:hessian_rmt}
}
where $M \in \reals^{(n-1) \times (n-1)}$ is a GOE matrix and $I \in \reals^{(n-1) \times (n-1)}$ is the identity.
\end{lemma}

We now compute bounds for $\norm{Y_\s}$ and $\norm{\Delta_\s}$ using the following two lemmas but let us describe a few quantities that we need before. It can be shown that the value of a $p$-spin glass Hamiltonian at local minima concentrates around a value
\beq{
    \label{eq:m_n_concentration}
    m_n = -n E_0 + \f{\log n}{2 \Th'(-E_0)} - K_0;
}
where $E_0$ is the lower energy bound in Fig.~\ref{fig:energy_barriers}, $\Th$ is the complexity of critical points discussed in Fig.~\ref{fig:energy_barriers}. In fact, the Hamiltonian is distributed as a negative Gumbel distribution with $m_n$ as the location parameter. The constants $C_0$ and $K_0$ are computed in~\citet[Eqn.~2.6]{subag2015extremal}. For a small $\d \in (0, p(E_0, E_{\trm{inf}}))$, define $\mathcal{V}(\d)$ to be the set of real symmetric matrices with eigenvalues in the interval
$$
    (p E_0 - 2 \sqrt{p(p-1)} - \d,\ p E_0 + 2 \sqrt{p(p-1)} + \d);
$$
and define the set $B_1 = \sqrt{n} \mathcal{V}(\d)$. The proof of Lem.~\ref{eq:g4_lemma17} hinges upon the observation in~\citet{subag2015complexity} that the second negative term (apart from the first GOE term) in the Hessian at a critical point~\eqref{eq:hessian_rmt} does not change the eigenvalue distribution of the GOE much.

\begin{lemma}
\label{lem:g3_lemma18}
For any $L > 0$ and
$$
    g_3(\s) = \D_\s + \f{\nu^2}{2 \sqrt{n}}\ \trace \cbrac{\rbrac{\nabla^2 f_\s(0)}^{-1}}, \quad B_3 = \rbrac{-n^{-\f{1}{2} + \e}, n^{-\f{1}{2} + \e}};
$$
there exists a sequence $e_{\e, \d}(n) \to 0$ such that
$$
    \P \rbrac{g_3 \notin B_3:\ \nabla^2 f_\s(0) = A_{n-1},\ \nabla f_\s(0) = 0,\ f_\s(0) = u} \leq e_{\e,\d}(n),
$$
for all $u \in n^{-1/2}\ (m_n - L, m_n + L)$ and $A_{n-1} \in \sqrt{n}\ \mathcal{V}(\d)$.
\end{lemma}
\begin{proof}
First note that $\nabla \lf_\s(0) = h$ and the term $\Delta_\s$ evaluates to $-\f{1}{2 \sqrt{n}}\ h^\top \rbrac{\nabla^2 f_\s(0)}^{-1} h$. The gradient of the perturbation $\nabla \lf_\s(0) = h$ is independent of the disorder that defines $\nabla^2 f_\s(0)$. Thus, conditional on $\nabla^2 f_\s(0) = A_{n-1} \in B_1$, we have that $\nabla \lf_\s(0)$ is distributed as
$$
    \nu\ \sum_{i=1}^{n-1}\ W_i\ a_i;
$$
where $W_i \sim N(0,1)$ are standard Gaussian random variables and $a_i$ are the eigenvectors of $A_{n-1}$. Because $h \sim N(0, \nu^2 I)$, we now see that $g_3(\s)$ has the same distribution as
$$
    \f{\nu^2}{2 \sqrt{n}}\ \sum_{i=1}^{n-1}\ (1 - W_i^2)\ \f{1}{\l(A_{n-1})};
$$
where $\l(A_{n-1})$ are the eigenvalues of $A_{n-1}$ and $W_i \sim N(0,1)$ are standard Gaussian random variables. This is a zero-mean random variable with bounded second moment, indeed
\beq{
    \rbrac{\f{\nu^2 C}{2 n}}^2\ \sum_{i=1}^{n-1}\ \E \cbrac{\rbrac{1 - W_i^2}^2} = 2 \rbrac{\f{\nu^2 C}{2 n}}^2 (n-1);
    \label{eq:second_moment_g3_proof}
}
for some constant $C > 0$. The lemma now follows by Chebyshev's inequality.
\end{proof}

\begin{remark}
Note that we had used a different normalization for the magnetic field in~\eqref{eq:tH_diff_normalization} as compared to~\eqref{eq:tH_general}. Switching back to the original normalization, i.e., replacing $\nu \leftarrow \nu n^{1/2}$, we see from~\eqref{eq:second_moment_g3_proof} that we still have a bounded second moment if $\nu$ is a constant, in particular, if it does not grow with $n$.
\end{remark}

\begin{lemma}[\citet{subag2015extremal}, Lem.~17]
\label{eq:g4_lemma17}
For $\tau_{\e,\d} \to 0$ as $n \to \infty$ slowly enough and $L > 0$
$$
    g_4(\s) = \f{\nu^2}{2 \sqrt{n}}\ \trace \cbrac{\rbrac{\nabla^2 f_\s(0)}^{-1}}, \quad B_4 = (C_0 - \tau_{\e,\d}(n),\ C_0 + \tau_{\e,\d}(n));
$$
we have
$$
    \lim_{n \to \infty}\ \E \cbrac{ \abs{\sqrt{n}\ \s: H(\s) \in [m_n - L, m_n + L],\ g_4(\s) \notin B_4,\ \nabla^2 f_\s(0) \in \sqrt{n} \mathcal{V}(\d) } } = 0.
$$
\end{lemma}

We can now easily combine the estimates of Lem.~\ref{lem:g3_lemma18} and Lem.~\ref{eq:g4_lemma17} to obtain:

\begin{theorem}
\label{thm:bound_Y_s_delta_s}
The perturbation of the normalized Hamiltonian is $\nu$ in the limit $n \to \infty$ while the perturbed local minima lie within a ball of radius $n^{-\a}$ with $\a \in (1/4,1/3)$ centered at the original local minima if $\nu$ is a constant. More precisely for all critical points $\s \in \crt_0(H)$,
\aeqs{
    \lim_{n \to \infty}\ \f{1}{\sqrt{n}}\ &\abs{\tf_\s(Y_\s) - f_\s(0)} = \nu,\\
    \norm{Y_\s} &\leq n^{-\a}.
}
\end{theorem}
\begin{proof}
The above two lemmas together imply that $\abs{g_4(\s)} \leq C_0$ with high probability and likewise, $\abs{\Delta_\s + g_4(\s)} \leq n^{-1/2 + \e}$. Together, this gives that
\aeq{
    \abs{\Delta_\s} &\leq C_0 + n^{-1/2 + \e}, \notag\\
    \sqrt{n} \abs{\tf_\s(Y_\s) - f_\s(0)} &\leq \norm{h} + C_0 + n^{-1/2 + \e}.
    \label{eq:bound_delta_s}
}
We now set $\nu \leftarrow \nu n^{1/2}$ in~\eqref{eq:tH_diff_normalization} to match the normalization used in the main paper to get that the perturbations to the normalized Hamiltonian are
\beq{
    \f{1}{\sqrt{n}} \abs{\tf_\s(Y_\s) - f_\s(0)} \leq \nu + \f{C_0}{n} + n^{-3/2 + \e}.
    \label{eq:norm_H_perturbations}
}
Since the gradient $\nabla f_\s(0)$ is independent of $\nabla^2 f_\s(0)$ and $f_\s(0) = u$, we have the following stochastic domination
$$
    \norm{Y_\s} \leq \f{c}{n}\ \norm{h};
$$
for some constant $c > 0$. The conditional probability in Lem.~\ref{lem:g3_lemma18} is now computed to be
\beq{
    \P \rbrac{Q_{n-1} \geq \f{n^{\f{1}{2}-\a}}{c \nu}};
    \label{eq:bound_Y_s}
}
where $Q_{n-1}$ is a standard Chi random variable with $n-1$ degrees of freedom, which goes to zero if $\nu$ is a constant that goes not grow with $n$. In particular, we have that $\norm{Y_\s} \leq n^{-a}$.
\end{proof}

\end{appendices}

\end{document}